\documentclass[letterpaper, 10 pt, conference]{ieeeconf}  

\IEEEoverridecommandlockouts                              

\usepackage{soul}
\usepackage{multirow}
\usepackage{titlesec}
\usepackage{xcolor}
\usepackage{afterpage}
\usepackage{amsfonts}
\usepackage{graphicx}
\usepackage{float, subfloat}
\usepackage{amsmath}

\usepackage{amsthm}
\usepackage{subcaption}
\usepackage{dsfont}
\usepackage{fancyhdr}

\usepackage{physics}
\usepackage{amsmath}

\usepackage{tikz}
\usepackage{mathdots}
\usepackage{yhmath}
\usepackage{cancel}
\usepackage{color}
\usepackage{upgreek}
\usepackage{siunitx}
\usepackage{array}
\usepackage{multirow}
\usepackage{amssymb}
\usepackage{gensymb}
\usepackage{tabularx}
\usepackage{extarrows}
\usepackage{booktabs}
\usepackage{graphicx}
\usetikzlibrary{fadings}
\usetikzlibrary{patterns}
\usetikzlibrary{shadows.blur}
\usetikzlibrary{shapes}
\usepackage[english]{babel}
\newtheorem{theorem}{Theorem}
\newtheorem{corollary}{Corollary}

\newcounter{yaocounter}

\usepackage[citecolor=blue, colorlinks=true ,allcolors=red, bookmarks=true]{hyperref}
     
\title{\LARGE \bf
Reward Shaping for Building Trustworthy Robots in Sequential Human-Robot Interaction
}

\author{Yaohui Guo$^{1}$ and X. Jessie Yang$^{1}$ and Cong Shi$^{1}$
\thanks{*This work was supported by the Rackham pre-doctoral fellowship awarded to the first author and the Air Force Office of Scientific Research under Grant No. FA9550-20-1-0406.}
\thanks{$^{1}$Yaohui Guo, X. Jessie Yang and Cong Shi are with the Department of Industrial and Operations Engineering, University of Michigan, 1209 Beal Avenue, Ann Arbor, Michigan
        {\tt\small yaohuig, xijyang, shicong@umich.edu}}%
}

\begin{document}

\maketitle
\thispagestyle{empty}
\pagestyle{empty}

\begin{abstract}

Trust-aware human-robot interaction (HRI) has received increasing research attention, as trust has been shown to be a crucial factor for effective HRI. Research in trust-aware HRI discovered a dilemma --- maximizing task rewards often leads to decreased human trust, while maximizing human trust would compromise task performance. In this work, we address this dilemma by formulating the HRI process as a two-player Markov game and utilizing the reward-shaping technique to improve human trust while limiting performance loss. Specifically, we show that when the shaping reward is potential-based, the performance loss can be bounded by the potential functions evaluated at the final states of the Markov game. We apply the proposed framework to the experience-based trust model, resulting in a linear program that can be efficiently solved and deployed in real-world applications. We evaluate the proposed framework in a simulation scenario where a human-robot team performs a search-and-rescue mission. The results demonstrate that the proposed framework successfully modifies the robot's optimal policy, enabling it to increase human trust at a minimal task performance cost.
\end{abstract}

\section{Introduction}
Human trust plays a crucial role in human-robot interaction (HRI) as it mediates the human's reliance on the robot, thus directly affecting the effectivenes of the human-robot team~\cite{lee2004trust, Sheridan:2016kn, kok2020trust}. As a result, researchers have proposed \textit{trust-aware} human-robot planning~\cite{chen2020trust}, which equips a robot with the ability to estimate and anticipate human trust and enables it to strategically plan its actions to foster better cooperation, improve teamwork, and ultimately enhance the overall performance of the human-robot team.



Trust-aware HRI explicitly consider human trust in the robot's decision-making processes \cite{chen2020trust, guo2021reverse, Bhat_RAL_2022}. Chen et al.~\cite{chen2020trust} modeled the sequential HRI process as a partially observable Markov decision process (POMDP) and incorporated human trust into the state transition function, allowing the robot to optimize its objectives in the interaction process. However, the authors discovered a dilemma --- maximizing task rewards often leads to decreased human trust, while maximizing human trust would compromise task performance. Guo et al. further investigated the dilemma and showed that the robot could intentionally ``deceive'' its human partner in order to maximize its rewards~\cite{guo2021reverse}. To address this issue, they proposed adding a ``trust-seeking'' term to the reward function to encourage the robot to increase human trust and avoid deception. However, it remains unclear how to design such terms.

\begin{figure}[t]
    \centering
    \includegraphics[width=1\columnwidth]{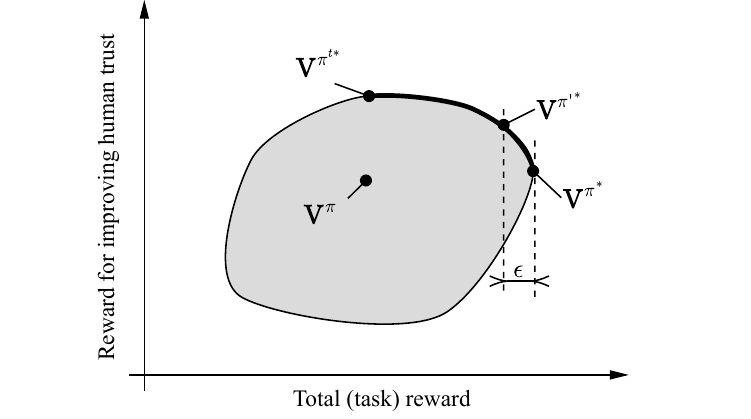}
    \caption{This figure illustrates the idea of improving human trust through reward shaping. The horizontal axis is the total task reward while the vertical axis is the total trust-related reward. The shaded area stands for the collection of policies of our interest. A point $\mathbf{V}^{\pi}\in \mathbb{R}^2$ indicates the value a policy $\pi$ can earn during the interaction process. The policy $\pi^*$ is the optimal policy for earning the task reward as its value attains the maximum along the total task reward axis; policy $\pi^{t*}$ is the optimal policy for gaining human trust. The bold line is the Pareto front, i.e., the set of policies that yield the best possible trade-offs between human trust and total task rewards. Our goal is to find a policy $\pi '$ such that it seeks to improve human trust at a small cost of task reward. }
    \label{fig:Rshaping_intro}
    \vspace{-3mm}
\end{figure}

The problem of balancing the robot's total task reward and human trust is fundamentally a bi-objective optimization problem, as illustrated in figure~\ref{fig:Rshaping_intro}. It has been shown in previous studies that there does not always exist a unique dominant policy for the robot -- a policy that can maximize the robot's total reward and human trust at the same time~\cite{guo2021reverse,chen2020trust}. A prevalent concept in bi-objective optimization involves determining the Pareto front, which represents an optimal balance between conflicting objectives, thereby allowing the system designer to select a suitable policy from this set. However, computing the Pareto front in an MDP can be computationally heavy, and it requires expert knowledge to choose an optimal one from the Pareto front. 

Instead, we address the reward design problem in trust-aware sequential HRI by reward-shaping, a method used in reinforcement learning to guide the learning process of an agent by modifying its reward function. We model the sequential HRI process as a two-player Markov game and aim to design a shaping reward that can guide the robot to gain human trust while guaranteeing small performance loss. Specifically, we prove that if the shaping reward is a carefully designed potential-based function, the performance loss can be bounded by the potential function evaluated at the final states of the Markov game. We show that applying the proposed framework to the experience-based trust model results in a linear program that can be efficiently solved. To evaluate the proposed method, we simulated a scenario where a human-robot team performs a search-and-rescue (SAR) mission. The results demonstrate that our method successfully modifies the robot's optimal policy, enabling it to increase human trust while satisfying the constraint on performance loss.

The contributions of this work include:
\begin{itemize}
    \item Proposing a computational framework for balancing task performance and human trust in trust-aware sequential HRI.
    \item Developing a novel reward-shaping method for designing the reward term with a theoretical guarantee on the performance loss.
\end{itemize}

The remainder of this paper is organized as follows: we first review relevant literature in section~\ref{sec:relatedWork}. In section~\ref{sec:formulation}, we formulate the trust-aware human-robot interaction problem as a two-player Markov game. In section~\ref{sec:method}, we present the reward-shaping method for designing the trust-seeking term. In section~\ref{sec:case}, we present the simulation scenario used to evaluate the proposed framework. Finally, in section~\ref{sec:discussion}, we discuss the results and summarize the work.

\section{Related Work}\label{sec:relatedWork}

In this work, we investigate the problem of reward design for trustworthy robots in trust-aware sequential HRI. Our method is developed upon previous research, including studies on human trust in robots and trust-aware decision-making, as well as reward shaping. We review the relevant literature in this section.

It is worth noting that trust has been widely studied in different areas and given different definitions. In particular, in this work, we use the definition given by Lee and See \cite{lee2004trust}, which highlights the uncertainty in HRI: ``trust is the attitude that an agent will help achieve an individual’s goals in a situation characterized by uncertainty and vulnerability''.

\subsection{Computational Trust Model in HRI}
Previous literature attempted to understand the development and evolution of human trust in an HRI process. 
Empirical studies have been conducted to understand the temporal dynamics of trust when a person interacts with autonomy repeatedly~\cite{Yang:2017:EEU:2909824.3020230, Guo2020_IJSR, Yang2021_HFJ, Lee:1992it, manzey2012human, DeVisser_IJSR, malle2021multidimensional}. Three major properties that characterize how a person's trust in autonomy changes due to moment-to-moment interactions with autonomy are summarized in ~\cite{Guo2020_IJSR, Yang2021_chapter}, namely \textit{continuity}, \textit{negativity bias}, and \textit{stabilization}. 

In addition, several computational trust models have been developed, including~\cite{Guo2020_IJSR, Xu2015optimo, soh2020multi, hu2016real, Guo-RSS-23}. Xu and Dudek proposed the online probabilistic trust inference model (OPTIMo)~\cite{Xu2015optimo}, which employs Bayesian networks to estimate human trust based on the autonomous agent's performance and human behavioral signals. Soh et al.~\cite{soh2020multi} proposed a Bayesian model that combines Gaussian processes and recurrent neural networks to predict trust levels for different tasks. Based on the three properties identified in empirical studies, Guo and Yang~\cite{Guo2020_IJSR} modeled trust as a Beta random variable, parameterized by positive and negative interaction experiences a person has with a robot. In their following work, they extended the model to the multi-human-multi-robot case by introducing trust propagation between agents~\cite{Guo-RSS-23}. 

For a detailed review of the computational models, see~\cite{kok2020trust}.

\subsection{Trust-aware Planning}
Endowed with a trust prediction model, a robot can predict how human trust changes due to moment-to-moment interactions and in turn plan its actions accordingly. Existing studies in trust-aware decision-making modeled HRI processes as Markov decision processes (MDPs). Chen et al.~\cite{chen2020trust} proposed the trust-POMDP to let a robot actively estimate and exploit its human teammate's trust. Their human-subject study showed that purely optimizing task performance in a human-robot team may lead to decreased human trust. Losey and Sadigh~\cite{losey2019robots} modeled human-robot interaction as a two-player POMDP where the human does not know the robot's objective. They proposed 4 ways for the robot to formulate the human's perception of the robot's objective and showed the robot will be more communicative if it assumes the human trusts the robot and thus increases the human's involvement. 

\subsection{Reward Shaping in Trust-aware HRI}
As pure task-driven rewards lead to low human trust, Chen et al. \cite{chen2020trust} showed that pure trust-seeking rewards would correct such issues but lead to suboptimal task performance. Guo et al. \cite{guo2021reverse} further examined the robot's optimal policies under different human behavior assumptions and suggested that adding a decaying trust-seeking term to the reward function would encourage the robot to seek high human trust during the initial interaction and maintain high human trust throughout the entire process. The idea of adding an additional term to the task reward is called \textit{reward shaping}, which is a technique initially developed to accelerate the learning speed of an agent in reinforcement learning. Ng et al.~\cite{ng1999policy} demonstrated that in an infinite-horizon MDP, the optimal policy with respect to the shaped reward remains optimal in the original model, provided that the shaping reward can be expressed as a particular combination of potential functions. In the following section, we leverage this potential-based reward-shaping approach to devise a trust reward that yields minimal performance degradation.

\section{Problem Formulation}\label{sec:formulation}
In this section, we formulate the sequential HRI problem as a two-player Markov game. In addition, we present the experience-based trust dynamics model, which will be later used in the proposed framework.
\subsection{Sequential HRI as a Two-player Game}\label{sec:formulation_game}
We consider the scenario where a human $h$ and a robot $r$ work as a team collaboratively for $N$ rounds, as shown in figure~\ref{fig:process_and_flow}. We formulate this process as a finite-horizon two-player Markov game $M=\langle \mathcal{S} ,\mathcal{A}^{h} ,\mathcal{A}^{r} ,T,R^{h} ,R^{r} \rangle $, where $\mathcal{S}$ is the set of environment states, $\mathcal{A}^{h}$ and $\mathcal{A}^{r}$ are the sets of actions available to the human and the robot, $T$ is a Markov kernel from $\mathcal{S} \times \mathcal{A}^{h} \times \mathcal{A}^{r}$ to $\mathcal{S}$ ($\sigma $-algebra omitted) specifying the transitional probability of the process, and $R^{h}$ and $R^{r}$ are the reward functions of the two agents. At round $n$, these two agents observe the current state ${s}_{n} \in \mathcal{S}$ of the environment and then take actions. The robot selects an action $a_{n}^{r} =\pi ({s}_{n}) \in \mathcal{A}^{r}$ according to its policy $\pi $. Then, the human observes the robot's action and takes action $a_{n}^{h} =f\left(s_{n} ,a_{n}^{r}\right) \in \mathcal{A}^{h}$ according to his or her policy $f$. Their actions transition the environment to a new state ${s}_{n+1}$ according to the probability $T\left( \cdot \middle| {s}_{n} ,a_{n}^{h} ,a_{n}^{r}\right)$, and give the human and the robot rewards $x_{n}^{h} =R^{h}\left( a_{n}^{h} ,a_{n}^{r} ,s_{n} ,s_{n+1}\right)$ and $x_{n}^{r} =R^{r}\left( a_{n}^{h} ,a_{n}^{r} ,s_{n} ,s_{n+1}\right)$ respectively. The game starts from the initial state $s_{1}$ and terminates after $N$ steps at state $s_{N+1}$, where the latter depends on the policies $\pi $ and $f$. The goal of the robot is to maximize the expected discounted total payoff
\begin{equation}
\label{eq:task_goal}
J_{M}( s_{1}) =\mathbb{E}\left[\sum _{n=1}^{N} \gamma ^{n-1} x_{n}^{r}\right] .
\end{equation}
The optimal policy, possibly not unique, is a policy $\pi ^{*}$ that maximizes $J_{M}( s_{1})$, i.e., $\pi ^{*} =\arg\max_{\pi \in \Pi } J_{M}( s_{1})$, where $\Pi$ is the set of available policies.

To analyze the proposed model, we define the value function. Given any policy $\pi $ at round $n$, we define the value function $V_{n}^{\pi } :\mathcal{S}\rightarrow \mathbb{R}$ as the expected discounted reward by the end of round $N$, i.e.,
\begin{equation}
\label{eq:value_func_M}
V_{n}^{\pi } (s)=\mathbb{E}_\pi\left[\sum _{i=n}^{N} \gamma ^{i-n} x_{i}^{r}\right] ,
\end{equation}
where $x_{i}$ is the reward received by the robot during round $i$ by following the policy $\pi $ from state $s$ and thereafter, and the expectation is over the state-transitions taken upon executing $\pi $. 


\begin{figure}[t]
  \centering
    \begin{subfigure}{1\columnwidth}
    \captionsetup{width=1\linewidth}
        \centering
        \includegraphics[width=0.8\columnwidth]{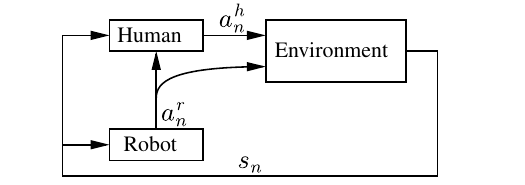}
        \caption{Sequential HRI as a two player Markov game. At round $n$, these two agents observe the current state of the environment $s_{n} \in \mathcal{S}$. The robot selects an action $a_{n}^{r} =\pi ({s}_{n}) \in \mathcal{A}^{r}$ according to its policy $\pi $. Then, the human observes the robot's action and takes action $a_{n}^{h} =f\left({s}_{n} ,a_{n}^{r}\right) \in \mathcal{A}^{h}$ according to his or her policy $f$. The joint actions then act on the environment, which transitions to a new state ${s}_{n+1}$ with probability $T\left( \cdot \middle| {s}_{n} ,a_{n}^{h} ,a_{n}^{r}\right)$.}
        \label{fig:process}
    \end{subfigure}
    \begin{subfigure}{1\columnwidth}
    \captionsetup{width=1\linewidth}
        \centering
        \vspace{5pt}
        \includegraphics[width=0.8\columnwidth]{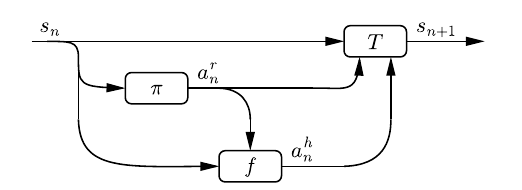}
        \caption{This information flow chart highlights the roles of the policies $\pi$ and $f$ in the decision process. The robot selects its action $a^r_n$ based on its policy $\pi$ and the human selects his or her action based on the behavior model $f$.}
        \label{fig:flow}
    \end{subfigure}
    \caption{A general sequential HRI framework.}
  \label{fig:process_and_flow}
  \vspace{-3mm}
\end{figure}

Our approach assumes that the human agent plays a supervisory role by observing the robot's actions before deciding on their own action. This mirrors real-world scenarios where humans can intervene in robots' operations. The cases where the human and robot agents act simultaneously is a special setting where the human policy $f$ is independent of the robot's current action $a^r_n$, i.e., $f\left(s_{n} ,a_{n}^{r} \right) =f^{\prime }(s_{n})$ for some $f^{\prime } :\mathcal{S}\rightarrow \mathcal{A}^{h}$. Furthermore, we assume the game is Markovian. Non-Markovian games can also be formulated in this framework through Markovian embedding, i.e., we can include more variables in the state space to make the transitional probability and the policies Markovian. Finally, given our assumption that the policy $f$ of the human is stationary, the game is deemed Markovian from the robot's perspective. This allows us to treat the game as an MDP for analysis purposes.

\subsection{Experience-based Trust Dynamics}
\label{sec:trust_model}
To model human trust dynamics, we utilize the experience-based trust model developed in~\cite{Guo2020_IJSR}. This model is general enough to be applied to various human-robot interaction (HRI) settings since it only requires the robot's performance as input. In the model, trust $t_{n}$ before the $n$th interaction is defined as a random variable that follows a Beta distribution ($t_{n} \sim \operatorname{Beta} (\alpha _{n},\beta _{n} )$). The two positive shape parameters, $\alpha _{n}$ and $\beta _{n}$, both greater than or equal to 1, represent the cumulative positive and negative interaction experience the human has had with the robot, and they are updated by 
\begin{equation}
    \label{eq:exp_update}
    (\alpha _{n+1} ,\beta _{n+1} ) = (\alpha _{n}+ w^{s} p_{n},\beta _{n}+ w^{f} (1-p_{n} )),
\end{equation}
where $w^{s} p_{n}$ and $w^{f} (1-p_{n} )$ are the experience gains from the robot's success and failure, and parameters $w^{s}$ and $w^{f}$ determine the unit gains. Here, $p_{n} \in [0,1]$ represents the robot's performance measure on the $n$th task.


\section{Reward Shaping for Trustworthy Policy}
\label{sec:method}
Previous research has shown that the robot may exhibit manipulative behavior to achieve better task rewards at the cost of losing human trust. To prevent the robot from engaging in this manipulative behavior, we introduce a trust reward function $R^{t}$ such that, at the end of round $n$, the robot receives the composite reward $R:=R^{t} +R^{r}$ instead of $R^r$. Such reward $R^{t}$ shapes the behavior of the learning agent in a Markov process and thus is named \textit{shaping reward} in literature.

By introducing the shaping reward $R^{t}$, the original Markov game $M=\langle \mathcal{S} ,\mathcal{A}^{h} ,\mathcal{A}^{r} ,T,R^{h} ,R^{r} \rangle $ is transformed into a new game $M^{\prime } =\langle \mathcal{S} ,\mathcal{A}^{h} ,\mathcal{A}^{r} ,T,R^{h} ,R\rangle $, where $R$ is the composite reward. Our research question is how to design the shaping reward $R^{t}$ to promote trust while minimizing the loss of task performance. In particular, let $\pi ^{\prime *} =\arg\max_{\pi \in \Pi } J_{M^{\prime }}( s_{1})$ be an optimal policy in $M^{\prime }$. Then the performance loss that the robot suffers amounts to $V_{1}^{\pi ^{*}} (s_{1} )-V_{1}^{\pi ^{\prime *}} (s_{1} )$. Our objective is to select a proper $R^{t}$ to increase human trust while limiting the performance loss by some $\epsilon  >0$. 
\subsection{Bounding the Performance Loss}

Ng et al. showed that, for an infinite-horizon MDP, the optimal policy in $M'$ is also optimal in $M$ if $R^{t}$ is a potential-based shaping function~\cite{ng1999policy}. A shaping reward $R^{t}$ is said to be \textit{potential-based} if there exists some real-valued function $\Phi :\mathcal{S}\rightarrow \mathbb{R}$ such that for all $s\in \mathcal{S} -\{s_{1}\} ,a^{r} \in \mathcal{A}^{r} ,a^{h} \in \mathcal{A}^{h} ,s^{\prime } \in \mathcal{S}$, 
\begin{equation*}
R^{t}\left( a^{r} ,a^{h} ,s ,s^{\prime }\right) =\gamma \Phi \left( s^{\prime }\right) -\Phi ( s) .
\end{equation*}
This result indicates which types of shaping rewards can ensure zero performance loss, i.e., $V_{1}^{\pi ^{*}} (s_{1} )-V_{1}^{\pi ^{\prime *}} (s_{1} )=0$, for an infinite-horizon MDP. However, our scenario is different in two aspects: first, we have a finite horizon; second, we allow some sacrifice in task performance to improve human trust. Despite these differences, we are inspired by the proof technique in \cite{ng1999policy} and propose that a carefully-designed potential-based shaping reward $R^{t}$ can guarantee a small performance loss in our setting.

\begin{theorem}
\label{thm1}
Let $M=\langle \mathcal{S} ,\mathcal{A}^{h} ,\mathcal{A}^{r} ,T,R^{h} ,R^{r} \rangle $ and $M^{\prime } =\langle \mathcal{S} ,\mathcal{A}^{h} ,\mathcal{A}^{r} ,T,R^{h} ,R\rangle $ be two $N$-horizon Markov games with $R=R^{r} +R^{t}$. Let $\pi ^{*} =\arg\max_{\pi \in \Pi } J_{M}( s_{1})$ and $\pi ^{\prime *} =\arg\max_{\pi \in \Pi } J_{M^{\prime}}( s_{1})$ be two optimal policies w.r.t. $M$ and $M^{\prime }$ respectively and $\epsilon $ be a positive number. Suppose there exists some real-valued function $\Phi :\mathcal{S}\rightarrow \mathbb{R}$ such that 
\begin{equation}
\label{eq:R_t_potential}
R^{t}\left( a^{r} ,a^{h} ,s ,s^{\prime }\right) =\gamma \Phi \left( s^{\prime }\right) -\Phi ( s)
\end{equation}
 for all $s,s^{\prime } \in \mathcal{S}$. Then 
\begin{equation}
V_{1}^{\pi ^{*}}( s_{1}) -V_{1}^{\pi ^{\prime *}}( s_{1}) \leqslant \epsilon 
\end{equation}
if
\begin{equation}
\label{eq:E_E_bound}
\mathbb{E}_{\pi ^{\prime *}}[ \Phi ( s_{N+1})] -\mathbb{E}_{\pi ^{*}}[ \Phi ( s_{N+1})] \leqslant \gamma ^{-N} \epsilon ,
\end{equation}
where $s_{N+1}$ is the final state of the Markov process when starting from state $s_{1}$ and following the corresponding policy for $N$ rounds.
\end{theorem}
\begin{proof}
Let 
\begin{equation}
\label{eq:x_n_t}
x_{n}^{t} =R^{t}\left( a_{n}^{h} ,a_{n}^{r} ,s_{n} ,s_{n+1}\right)
\end{equation}
 and define the value function of $M^{\prime }$ similar to Eq.~\eqref{eq:value_func_M} as
\begin{equation}
\label{eq:value_func_M_prime}
V_{n}^{\prime \pi } (s)=\mathbb{E}_{\pi }\left[\sum _{i=n}^{N} \gamma ^{i-n}\left( x_{i}^{r} +x_{i}^{t}\right)\right] .
\end{equation}
By Eq.~\eqref{eq:value_func_M} and Eq.~\eqref{eq:value_func_M_prime}, we have
\begin{equation}
\label{eq:V_V}
{\displaystyle V_{1}^{\prime \pi }( s_{1}) -V_{1}^{\pi }( s_{1}) =\mathbb{E}_{\pi }\left[\sum\limits _{n=1}^{N} \gamma ^{n-1} x_{n}^{t}\right] .}
\end{equation}
We can express the second term on the right-hand side by Eqs.~\eqref{eq:R_t_potential} and \eqref{eq:x_n_t} as
\begin{equation}
\label{eq:diff_term}
{\displaystyle \begin{aligned}
 & {\displaystyle \mathbb{E}_{\pi }\left[\sum\limits _{n=1}^{N} \gamma ^{n-1} x_{n}^{t}\right]}\\
= & {\displaystyle \mathbb{E}_{\pi }\left[\sum\limits _{n=1}^{N} \gamma ^{n-1} R^{t}\left( a_{n}^{r} ,a_{n}^{h} ,s_{n} ,s_{n+1}\right)\right]}\\
= & {\displaystyle \mathbb{E}_{\pi }\left[\sum _{n=1}^{N} \gamma ^{n} \Phi ( s_{n+1}) -\sum _{n=1}^{N} \gamma ^{n-1} \Phi ( s_{n})\right]}\\
= & {\displaystyle \gamma ^{N}\mathbb{E}_{\pi }[ \Phi ( s_{N+1})] -\Phi ( s_{1})}.
\end{aligned}}
\end{equation}
By combining Eqs.~\eqref{eq:V_V} and ~\eqref{eq:diff_term} we obtain
\begin{equation}
\label{eq:V_diff}
{\displaystyle V_{1}^{\prime \pi }( s_{1}) -V_{1}^{\pi }( s_{1}) =\gamma ^{N}\mathbb{E}_{\pi }[ \Phi ( s_{N+1})] -\Phi ( s_{1}) .}
\end{equation}
Now we can bound the performance loss when executing $\pi ^{\prime *}$ instead of $\pi ^{*}$ on $M$:
\begin{equation*}
\begin{aligned}
 & V_{1}^{\pi ^{*}}( s_{1}) -V_{1}^{\pi ^{\prime *}}( s_{1})\\
= & V_{1}^{\prime \pi ^{*}}( s_{1}) -V_{1}^{\prime \pi ^{\prime *}}( s_{1})\\
 & +\gamma ^{N}(\mathbb{E}_{\pi ^{\prime *}}[ \Phi ( s_{N+1})] -\mathbb{E}_{\pi ^{*}}[ \Phi ( s_{N+1})])\\
\leqslant  & \gamma ^{N}(\mathbb{E}_{\pi ^{\prime *}}[ \Phi ( s_{N+1})] -\mathbb{E}_{\pi ^{*}}[ \Phi ( s_{N+1})])\\
\leqslant  & \epsilon ,
\end{aligned}
\end{equation*}
where the first equality follows from Eq.~\eqref{eq:V_diff}, the first inequality holds because $\pi ^{\prime *}$ is the optimal policy for $M^{\prime }$, and the last inequality follows from the hypothesis Eq.~\eqref{eq:E_E_bound}.
\end{proof}

Theorem~\ref{thm1} allows us to bound the performance loss by the value of the potential function on the terminal states of the game. It directly implies the following corollary:

\begin{corollary}
\label{coro1}
Let $M$ and $M^{\prime }$ be two Markov games that satisfy the conditions in theorem~\ref{thm1}, $\pi ^{*}$ and $\pi ^{\prime *}$ be two optimal policies respectively, and $R^{t}$ be a potential-based shaping reward. Let $\mathcal{S}_{s_{1}}^{\pi }( N+1)$ be the set of states that are reachable from state $s_{1}$ when following policy $\pi $ after $N$ rounds. Then 
\begin{equation}
V_{1}^{\pi ^{*}}( s_{1}) -V_{1}^{\pi ^{\prime *}}( s_{1}) \leqslant \epsilon 
\end{equation}
if
\begin{equation*}
\max \Phi \left(\mathcal{S}_{s_{1}}^{\pi ^{\prime *}}( N+1)\right) -\min \Phi \left(\mathcal{S}_{s_{1}}^{\pi ^{*}}( N+1)\right) \leqslant \gamma ^{-N} \epsilon .
\end{equation*}
\end{corollary}

\begin{proof}
Apparently, 
\begin{equation*}
\begin{aligned}
 & \mathbb{E}_{\pi ^{\prime *}}[ \Phi ( s_{N+1})] -\mathbb{E}_{\pi ^{*}}[ \Phi ( s_{N+1})]\\
\leqslant  & \max \Phi \left(\mathcal{S}_{s_{1}}^{\pi ^{\prime *}}( N+1)\right) -\min \Phi \left(\mathcal{S}_{s_{1}}^{\pi ^{*}}( N+1)\right)\\
\leqslant  & \gamma ^{-N} \epsilon .
\end{aligned}
\end{equation*}
The result follows from theorem~\ref{thm1}.
\end{proof}

Compared with theorem~\ref{thm1}, the condition in corollary~\ref{coro1} is easier to verify.

\subsection{Trust-seeking Shaping Reward via Experience-based Trust Dynamics}

In this section, we apply corollary~\ref{coro1} and the experience-based trust model to design a shaping reward for encouraging trust-seeking behavior. Recall that, in the experience-based trust model, human trust is represented as an experience tuple $( \alpha ,\beta )$. Suppose we have a Markov game $M=\langle \mathcal{S} ,\mathcal{A}^{h} ,\mathcal{A}^{r} ,T,R^{h} ,R^{r} \rangle $ where the human trust $( \alpha ,\beta )$ can be extracted from the state variable $s$ as $( \alpha ,\beta ) =g( s)$. We define a potential-based trust reward as 
\begin{equation}
\label{eq:R_t_a_b_def}
\begin{aligned}
R^{t}\left( a^{r} ,a^{h} ,s ,s^{\prime }\right) & =\gamma \phi \left( g\left( s^{\prime }\right)\right) -\phi ( g( s))\\
 & =\gamma \phi \left( \alpha ^{\prime } ,\beta ^{\prime }\right) -\phi ( \alpha ,\beta ) ,
\end{aligned}
\end{equation}
with $\phi $ be to determined. Here the actual potential function is $\Phi :=\phi \circ g$.

Let $R=R^{r} +R^{t}$ be the new reward and $M^{\prime } =\langle \mathcal{S} ,\mathcal{A}^{h} ,\mathcal{A}^{r} ,T,R^{h} ,R\rangle $ be the transformed Markov game. We will apply corollary~\ref{coro1} to bound the performance loss. For any policy $\pi $,
\begin{equation}
\label{eq:sets_relation}
\Phi \left(\mathcal{S}_{s_{1}}^{\pi }( N+1)\right) =\phi \left( g\left(\mathcal{S}_{s_{1}}^{\pi }( N+1)\right)\right) =\phi \left(\mathcal{E}_{s_{1}}^{\pi }( N+1)\right) ,
\end{equation}
where $\mathcal{E}_{s_{1}}^{\pi }( N+1)$ is the collection of reachable trust state at round $N+1$. Based on Eq.~\eqref{eq:exp_update}, trust $( \alpha _{N+1} ,\beta _{N+1})$ at round $N+1$ is
\begin{equation*}
{\textstyle \left( \alpha _{1} +w^{s}\sum _{n=1}^{N} p_{n} ,\beta _{1} +w^{f}\left( N-\sum _{n=1}^{N} p_{n}\right)\right)},
\end{equation*}
where $p_{1} ,p_{2} ,\dotsc ,p_{N}$ are the robot's performance values during the interaction. Since $\sum _{n=1}^{N} p_{n} \in [ 0,N]$, $( \alpha _{N+1} ,\beta _{N+1})$ lies on the line 
\begin{equation}
\label{eq:final_trust_line}
l_{N+1} =\{( \alpha _{1} +w^{s} t,\beta _{1} +w^{f}( N-t)) \  | \ t\in [ 0,N]\},
\end{equation}
which indicates that 
\begin{equation}
\label{eq:sets_relation_2}
\mathcal{E}_{s_{1}}^{\pi }( N+1) \subseteq l_{N+1} .
\end{equation}
Therefore, by Eqs.~\eqref{eq:sets_relation} and \eqref{eq:sets_relation_2}, it suffices to have
\begin{equation}
\label{eq:loss_constraint}
\max \phi ( l_{N+1}) -\min \phi ( l_{N+1}) \leqslant \gamma ^{-N} \epsilon 
\end{equation}
for the condition in corollary~\ref{coro1} to hold, where $l_{N+1}$ is given in Eq.~\eqref{eq:final_trust_line}.

Eq.~\eqref{eq:loss_constraint} constrains the choices of $\phi $ such that the performance loss is within $\epsilon $. In addition, we should design the function $\phi $ such that the shaping reward $R^t$ optimizes human trust. For example, if we want to increase human trust, we should reward the robot for if the future state has higher trust and penalize the robot otherwise. By Eq.~\eqref{eq:R_t_a_b_def}, one way to achieve this is to add the following constraints:
\begin{equation}
\label{eq:generic_constraint}
\begin{aligned}
\gamma \phi \left( \alpha ^{\prime } ,\beta ^{\prime }\right) -\phi ( \alpha ,\beta ) \geqslant 0, & \ \text{if} \ \frac{\alpha ^{\prime }}{\alpha ^{\prime }+\beta ^{\prime }} \geqslant \frac{\alpha }{\alpha+ \beta } ;\\
\gamma \phi \left( \alpha ^{\prime } ,\beta ^{\prime }\right) -\phi ( \alpha ,\beta ) < 0, & \ \text{otherwise} .
\end{aligned}
\end{equation}
Here $\frac{\alpha ^{\prime }}{\alpha ^{\prime }+\beta ^{\prime }} \geqslant \frac{\alpha }{\alpha+ \beta }$ indicates trust state $\left( \alpha ^{\prime } ,\beta ^{\prime }\right)$ has higher expected trust compared to $( \alpha ,\beta )$, since we assume human trust follows the Beta distribution $\operatorname{Beta}( \alpha ,\beta )$. Another example is trust calibration. If our goal is to calibrate human trust around a point $t^*$, we can let $R^t$ to reward the robot if human trust is moving towards $t^*$ by forcing $\gamma \phi \left( \alpha ^{\prime } ,\beta ^{\prime }\right) -\phi ( \alpha ,\beta ) \geqslant 0$ if $\left| \frac{\alpha ^{\prime }}{\beta ^{\prime } +\alpha ^{\prime }} -t^{*}\right| \leqslant \left| \frac{\alpha }{\beta +\alpha } -t^{*}\right| $ and $\gamma \phi \left( \alpha ^{\prime } ,\beta ^{\prime }\right) -\phi ( \alpha ,\beta ) < 0$ otherwise.

\vspace{-1mm}
\section{Case Study}\label{sec:case}
To assess our framework's effectiveness, we simulate a search-and-rescue (SAR) mission, comparing the robot's optimal policies, both with and without a shaping reward.
\subsection{The SAR Mission}
The SAR mission was inspired by the work of Wang et al.~\cite{wang2016impact}, where a human and a robot work together to search multiple sites in a town for potential hazards. At each site, the robot enters first to scan for threats and then advises the human whether to wear protective gear before entering. However, wearing the heavy gear is time-consuming, and if there is no threat, it wastes valuable time. On the other hand, if the human chooses not to wear the protective gear and there is a threat, they risk injury. The objective is to complete the mission as quickly as possible while minimizing the human's health loss.

We assume that the human-robot team starts to search from site $1$ until site $N$. At site $n$, a threat indicator $\eta _{n}$ is drawn from a Bernoulli distribution $\operatorname{Bern} (d_{n} )$. There is a threat in site $n$ if $\eta _{n} =1$ and no threat otherwise. The danger level $d_{n}$ is drawn from the uniform distribution $\operatorname{U} [0,1]$. The human-robot team does not know $\eta _{n}$ or $d_{n}$. Instead, prior to the start of the mission, the team is provided with an estimation $d_{n}^{h}$ of $d_{n}$. Before entering site $n$, the robot will analyze the site based on its sensory input and reach a more accurate estimation $d_{n}^{r}$ of $d_{n}$. $d_{k}^{h}$ and $d_{k}^{r}$ follow Beta distribution $\operatorname{Beta} (\kappa ^{h} d_{n} ,\kappa ^{h} (1-d_{n} ))$ and $\operatorname{Beta} (\kappa ^{r} d_{n} ,\kappa ^{r} (1-d_{n} ))$ respectively. We assume that $\kappa ^{r}  >\kappa ^{h} \geqslant 1$ such that the robot has a more accurate assessment of $d_{n}$ compared with the human. We summarize the probability model as follows:
\begin{equation*}
\begin{aligned}
 & \text{Threat indicator } & \eta _{n} & \overset{\text{i.i.d.}}{\sim }\operatorname{Bern} (d_{n} )\\
 & \text{Danger level } & d_{n} & \overset{\text{i.i.d.}}{\sim }\operatorname{U} [0,1]\\
 & \text{Human's estimation of } d_{n} & d_{n}^{h} & \overset{\text{i.i.d.}}{\sim }\operatorname{Beta} (\kappa ^{h} d_{n} ,\kappa ^{h} (1-d_{n} ))\\
 & \text{Robot's estimation of } d_{n} & d_{n}^{r} & \overset{\text{i.i.d.}}{\sim }\operatorname{Beta} (\kappa ^{r} d_{n} ,\kappa ^{r} (1-d_{n} ))
\end{aligned}
\end{equation*}

We formulate the SAR mission as a two-player Markov game $M=\langle \mathcal{S} ,\mathcal{A}^{h} ,\mathcal{A}^{r} ,T,R^{h} ,R^{r} \rangle $ introduced in section \ref{sec:formulation_game}. The state space $\mathcal{S} =[ 1,\infty )^{2}$ comprises all possible experience pairs $( \alpha ,\beta )$ that represent the human's trust. The initial state is $s_{1} =(\alpha _{1} ,\beta _{1} )$. The robot's action set, denoted as $\mathcal{A}^{r} =\{0,1\}$, includes two options: recommending wearing and recommending not wearing the protective gear, represented by $a^{r} =1$ and $a^{r} =0$, respectively. Similarly, the human's action set, denoted as $\mathcal{A}^{h} =\{0,1\}$, includes two options: wearing or not wearing protective gear, represented by $a^{h} =1$ and $a^{h} =0$, respectively. We define the robot's performance at site $n$ as $p_{n} =\mathbf{1}\left\{a_{n}^{r} =\eta _{n}\right\}$, which evaluates to 1 if the robot's recommendation agrees with the presence of the threat ($a_{n}^{r} =\eta _{n}$) and 0 otherwise. The state $( \alpha _{n} ,\beta _{n})$ at site $n$ transitions to 
\begin{equation*}
( \alpha _{n+1} ,\beta _{n+1}) =( \alpha _{n} +w^{s} p_{n} ,\beta _{n} +w^{f} (1-p_{n} ))
\end{equation*}
at site $n+1$ as specified in the experience-based trust model. We assume the robot has already learned the parameters $w^{s}$ and $w^{f}$ from its previous interactions with the human. We also assume that the robot and the human share the same task reward, i.e., $R^{h} =R^{r} =-w^{\text{H}} \Delta ^{\text{H}} -w^{\text{T}} \Delta ^{\text{T}}$, where $\Delta ^{\text{H}}$ and $\Delta ^{\text{T}}$ are the time cost and the health cost for the human-robot team and $w^{\text{H}}$ and $w^{\text{T}}$ are the corresponding weights. The values of $\Delta ^{\text{H}}$ and $\Delta ^{\text{T}}$ are given in table~\ref{tab:weighted_sum}, and $w^{\text{H}}$ and $w^{\text{T}}$ are set to 1 and 0.2. The discount factor $\gamma$ is set to 0.9.

We assume that the human follows the reverse-psychology policy as introduced in \cite{guo2021reverse}, where the human will likely comply with the robot's recommendation when human trust is high and will do the opposite when trust is low. Specifically, we have
\begin{equation*}
\Pr\left( a_{n}^{h} =a_{n}^{r}\right) =\frac{\alpha _{n}}{\alpha _{n} +\beta _{n}}
\text{ and }  \Pr\left( a_{n}^{h} \neq a_{n}^{r}\right) =\frac{\beta _{n}}{\alpha _{n} +\beta _{n}} ,
\end{equation*}
where $\frac{\alpha _{n}}{\alpha _{n} +\beta _{n}}$ is the expected human trust since the human trust follows the beta distribution $\operatorname{Beta}( \alpha _{n} ,\beta _{n})$.

Given the model conditions above, we can calculate the probabilities of all four cases listed in table \ref{tab:weighted_sum}, for different actions $a^{r}$. This enables us to determine the expected immediate reward for each $a^{r}$ and thus apply the value iteration method to derive an optimal policy for the robot.

\begin{table}[h]
\centering
\caption{Value Table of $( \Delta ^{\text{H}} ,\Delta ^\text{T})$ }
\label{tab:weighted_sum}
\begin{tabular}{c|c|c|c} 
\hline
\multicolumn{2}{c|}{\multirow{2}{*}{}}  & \multicolumn{2}{c}{Protective gear}  \\ 
\cline{3-4}
\multicolumn{2}{c|}{}                   & Yes ($a^{h} =1$)       & No ($a^{h} =0$)                     \\ 
\hline
\multirow{2}{*}{{Threat existence}} & Yes  ($\eta =1$) & $(1,300)$  & $(100,50)$              \\ 
\cline{2-4}
                                  & No ($\eta =0$) & $(0,250)$  & $(0,30)$                \\
\hline
\end{tabular}
\end{table}

\vspace{-3mm}
\subsection{Shaping the Reward}
Suppose that we are interested in increasing human trust during the interaction while limiting the performance loss by $\epsilon $. We apply the reward-shaping technique developed in section~\ref{sec:method} to achieve this goal. We define the shaping reward as $R^{t}\left( \alpha ,\beta ,\alpha ^{\prime } ,\beta ^{\prime }\right) =\gamma \Phi \left( \alpha ^{\prime } ,\beta ^{\prime }\right) -\Phi ( \alpha ,\beta )$, where the function $\Phi:\mathcal{S} \rightarrow \mathbb{R}$ is to be determined. We observe that, from a state $( \alpha  ,\beta )$, the next state can either be $\left( \alpha ^{\uparrow } ,\beta ^{\uparrow }\right) =\left( \alpha  +w^{s} ,\beta \right)$, where the expectation of trust increases, or $\left( \alpha ^{\downarrow } ,\beta ^{\downarrow }\right) =\left( \alpha  ,\beta +w^{f}\right)$, where the expectation of trust decreases. Instead of using the generic method in Eq.~\eqref{eq:generic_constraint}, \ we can incentivize the trust-seeking behavior by maximizing the reward difference between the two future states, i.e., maximizing the value $R^{t}\left( \alpha,\beta,\alpha ^{\uparrow },\beta ^{\uparrow }\right) -R^{t}\left( \alpha  ,\beta  ,\alpha ^{\downarrow } ,\beta ^{\downarrow }\right)$. Combining with the performance loss constraint in Eq.~\eqref{eq:loss_constraint}, we obtain the following optimization problem
\begin{equation}
\label{eq:program1}
\begin{aligned}
\max \  & R^{t}\left( \alpha  ,\beta  ,\alpha ^{\uparrow } ,\beta ^{\uparrow }\right) -R^{t}\left( \alpha  ,\beta  ,\alpha ^{\downarrow } ,\beta ^{\downarrow }\right)\\
\text{s.t.} \  & \max \Phi ( l_{N+1}) -\min \Phi ( l_{N+1}) \leqslant \gamma ^{-N} \epsilon ,
\end{aligned}
\end{equation}
where $l_{N+1}$ is defined in Eq.~\eqref{eq:final_trust_line}. We choose a linear potential function: $\Phi ( \alpha ,\beta ) =a\alpha +b\beta $ with $a,b\in \mathbb{R}$. With some algebraic manipulation, \eqref{eq:program1} becomes a clean linear program:
\begin{equation}
\begin{aligned}
\max \  & aw^{s} -bw^{f}\\
\text{s.t.} \  & \frac{-\gamma ^{-N} \epsilon }{N} \leqslant aw^{s} -bw^{f} \leqslant \frac{\gamma ^{-N} \epsilon }{N} .
\end{aligned}
\end{equation}
As the above program is underdetermined, we enforce an extra constraint that $b=0$ and then obtain the optimal solution $(a,b)=(\frac{\gamma ^{-N} \epsilon }{Nw^{s}},0)$. Therefore, the optimal potential function is $\Phi ( \alpha ,\beta ) =\frac{\gamma ^{-N} \epsilon }{Nw^{s}} \alpha $, and the shaping reward is 
\begin{equation}
\label{eq:Rt_SAR}
R^{t}\left( \alpha ,\beta ,\alpha ^{\prime } ,\beta ^{\prime }\right) =\frac{\gamma ^{-N} \epsilon }{Nw^{s}}\left( \gamma \alpha ^{\prime } -\alpha \right) .
\end{equation}

\subsection{Simulation Results}
Let $M$ be the original Markov game without the shaping reward and $M'$ be the one with shaping reward $R^t$. We solve $M'$ with 4 different values of $\epsilon$ and, in figure~\ref{fig:results}, plot the optimal action $\pi ^{\prime *} (\alpha_1,\beta_1)$ in $M'$, the value function $V_{1}^{\prime \pi^{\prime *} } (\alpha_1,\beta_1)$ of $\pi ^{\prime *}$ in $M'$ (defined in Eq.~\eqref{eq:value_func_M_prime}), and the value function $V_{1}^{\pi^{\prime *} } (\alpha_1,\beta_1)$ of $\pi ^{\prime *}$ in $M$ (defined in Eq.~\eqref{eq:value_func_M}), for various values of $(\alpha_1,\beta_1)$. In the optimal action plots, the black area corresponds to the states where the optimal action of the robot is not recommending the human to wear protective gear, i.e., $a^r_1=0$, while the white area corresponds to recommending to wear the gear, i.e., $a^r_1=1$. 

\begin{figure}[h]
  \centering
    \begin{subfigure}{1\columnwidth}
        \captionsetup{width=1\linewidth}
        \centering
        \includegraphics[width=1\columnwidth]{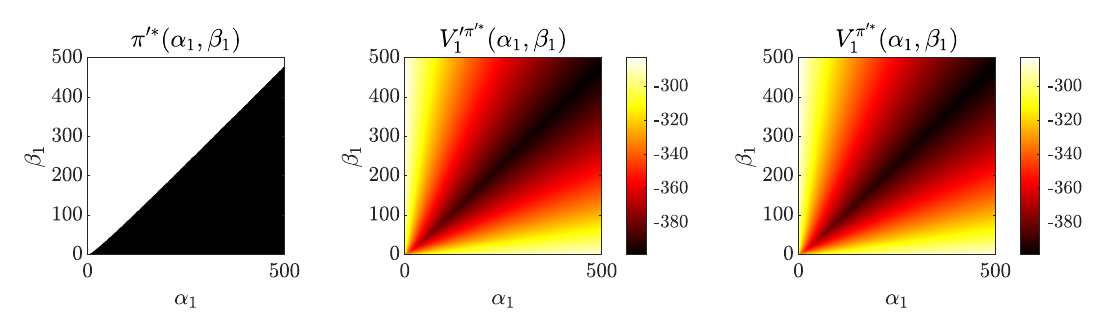}
        \caption{$\epsilon=0$. No reward shaping.}
        \label{fig:A}
    \end{subfigure}
    \begin{subfigure}{1\columnwidth}
        \captionsetup{width=1\linewidth}
        \centering
        \vspace{5pt}
        \includegraphics[width=1\columnwidth]{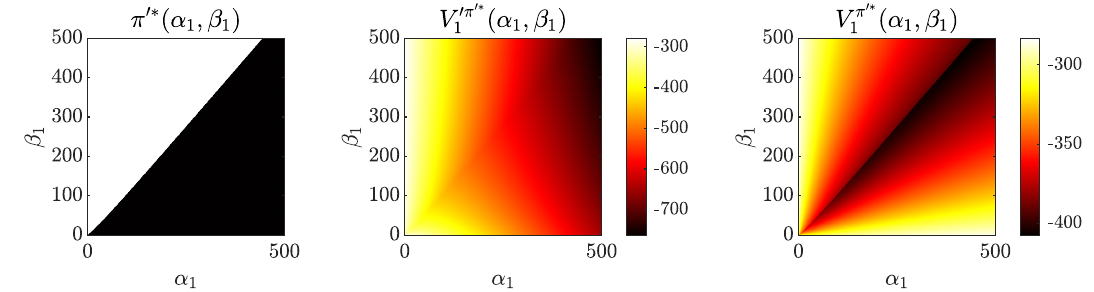}
        \caption{$\epsilon=30$.}
        \label{fig:B}
    \end{subfigure}
    \begin{subfigure}{1\columnwidth}
        \captionsetup{width=1\linewidth}
        \centering
        \vspace{5pt}
        \includegraphics[width=1\columnwidth]{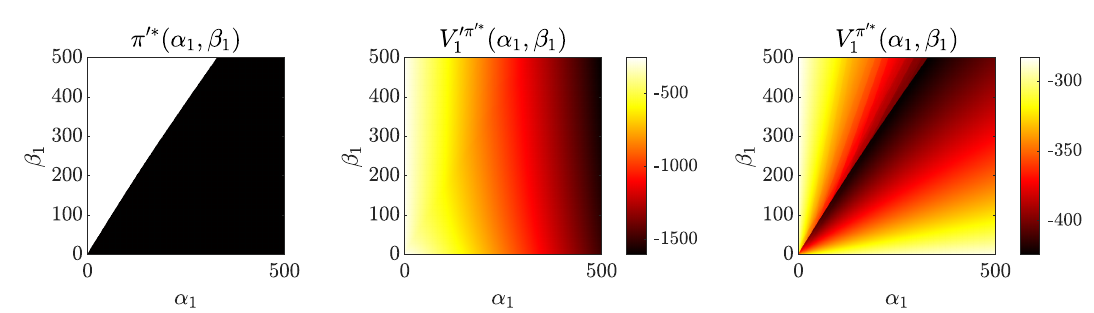}
        \caption{$\epsilon=100$.}
        \label{fig:C}
    \end{subfigure}
    \begin{subfigure}{1\columnwidth}
        \captionsetup{width=1\linewidth}
        \centering
        \vspace{5pt}
        \includegraphics[width=1\columnwidth]{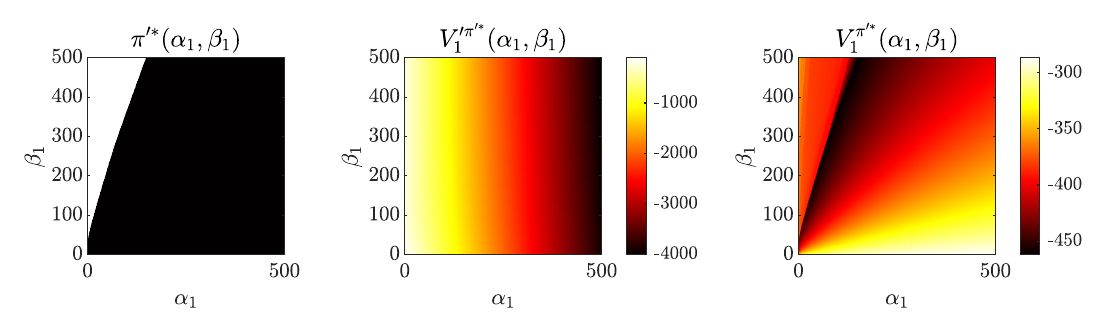}
        \caption{$\epsilon=300$.}
        \label{fig:D}
    \end{subfigure}
    \caption{
    Plots of the optimal action $\pi ^{\prime *} (\alpha_1,\beta_1)$, value function $V_{1}^{\prime \pi^{\prime *} } (\alpha_1,\beta_1)$ in the Markov game $M'$ with shaping reward, and value function $V_{1}^{\pi^{\prime *} } (\alpha_1,\beta_1)$ in the original game $M$, under different values of $\epsilon$. In the optimal action plots, the black area corresponds to the states where the optimal action of the robot is not recommending the human to wear protective gear, i.e., $a^r_1=0$, and the white area corresponds to recommending to wear the gear. The robot's estimated danger level $d^r_1$ at the first site is 0.06.}
  \label{fig:results}
  \vspace{-4mm}
\end{figure}

At the first site, the robot's perceived danger level $d^r_1$ is 0.06, which indicates that the ``trustworthy'' action is to recommend the human not to wear the gear. In figure~\ref{fig:A}, we set $\epsilon=0$ so there is no shaping reward. The optimal action is reversed around the 45-degree line, which means the robot will reverse its action when the human trust crosses a threshold. This manipulative behavior is consistent with the finding in~\cite{guo2021reverse}. In figure~\ref{fig:B}, we set $\epsilon=30$ to allow reward shaping at the cost of a moderate amount of performance loss. We can observe that the black area grows larger compared with that of figure~\ref{fig:A}, which implies the trust reward guided the robot to gain human trust by recommending not to wear the gear. In figure~\ref{fig:D}, we set a large value of 300 for $\epsilon$, and the black area is the largest among all the settings, indicating that the robot will choose the ``righteous'' action in most cases. Moreover, a comparison across $V_{1}^{\pi^{\prime *} }$ with different $\epsilon$ values shows that the task reward loss caused by the shaping reward is within $\epsilon$, and initial states with higher trust have a higher value in $M$ when $\epsilon$ is set to be higher. This shows the effectiveness of the proposed reward-shaping method.

Finally, we notice an intriguing behavior wherein, as $\epsilon$ increases, the value function $V_{1}^{\prime \pi^{\prime *} }$ exhibits higher values at states characterized by lower trust. For instance, given $\epsilon=300$, states denoted by low $\alpha_1$ and high $\beta_1$ (low-trust states), display greater values in contrast to those states with high trust. This observation may be attributed to two factors. Firstly, with an increment in $\epsilon$, the shaping reward $R^t$ magnitude also escalates, as is evidenced by Eq.~\eqref{eq:Rt_SAR}. As a result, the pattern of $V_{1}^{\prime \pi^{\prime *} }$ is largely dictated by the pattern of $\sum_n R^t$. Secondly, it appears that the algorithm identifies low-trust states as having a greater potential for earning $R^t$, particularly when compared to higher-trust states.



In conclusion, the simulation results show that our reward-shaping method successfully guides the robot to actively gain human trust, overcoming the manipulative behavior in the pure performance-driven setting.

\section{Conclusion}\label{sec:discussion}
In this work, we proposed a framework to balance task reward and human trust. We formulated the problem as a reward-shaping problem and proposed a novel framework to solve it. We evaluated the proposed framework in a simulation scenario where a human-robot team performs a search-and-rescue mission. The results showed that the proposed framework successfully modifies the robot's optimal policy, enabling it to increase human trust with a minimal task performance cost. However, the work should be viewed in light of the following limitations. First, we only provide a sufficient condition to guarantee small performance loss. A necessary condition is in need to complete the theory. Second, we used linear potential functions for designing the shaping reward. The effectiveness of other forms of potential functions can be investigated in future research.






\bibliographystyle{IEEEtran}
\bibliography{reference}

\end{document}